\documentclass{article} 

\usepackage{algorithm}
\usepackage{algorithmic}

\usepackage{graphicx} 

\usepackage{amsmath,amsthm,amsfonts,amssymb,mathrsfs}
\usepackage[fleqn,tbtags]{mathtools}

\newcommand{\beq}{\begin{equation}} 
\newcommand{\eeq}{\end{equation}}

\newcommand{\reg}{{\omega}}

\newcommand{\bi}{\begin{itemize}}
\newcommand{\be}{\begin{enumerate}}
\newcommand{\ei}{\end{itemize}}
\newcommand{\ee}{\end{enumerate}}

\newcommand{\R}{{\mathbb R}}

\newcommand{\calG}{{\cal G}}

\newcommand{\calN}{{\cal N}}

\newcommand{\lb}{{\langle}}
\newcommand{\rb}{{\rangle}}

\def\boldf#1{\hbox{\rlap{$#1$}\kern.4pt{$#1$}}}

\newcommand{\trans}{^{\scriptscriptstyle \top}}

\newtheorem{proposition}{Proposition}[section]

\newtheorem{definition}{Definition}[section]

\newcommand{\rd}{\R^d}

\DeclareMathOperator\prox{prox}
\DeclareMathOperator\argmin{argmin}

\newcommand{\down}{^\downarrow}
\DeclareMathOperator\supp{supp}
\DeclareMathOperator\conv{conv}

\mathtoolsset{showonlyrefs,showmanualtags}

\def\eop{$\rule{1.3ex}{1.3ex}$}
\renewcommand\qedsymbol\eop  

\title{Sparse Prediction with the $k$-Support Norm}

\author{Andreas Argyriou \\ argyriou@ttic.edu \\ Toyota Technological Institute at Chicago \bigskip
\\ Rina Foygel \\ rina@uchicago.edu \\ University of Chicago \bigskip
\\ Nathan Srebro \\ nati@ttic.edu \\ Toyota Technological Institute at Chicago}

\begin{document}

\maketitle

\begin{abstract}
  We derive a novel norm that corresponds to the tightest convex
  relaxation of sparsity combined with an $\ell_2$ penalty. We show
  that this new {\em $k$-support norm} provides a tighter relaxation than the elastic
  net and is thus a good replacement for the Lasso or the elastic net
  in sparse prediction problems.  Through the study of the $k$-support
  norm, we also bound the looseness of the elastic net, thus shedding new
  light on it and providing justification for its use.
\end{abstract}

\section{Introduction}
\label{sec:intro}


Regularizing with the $\ell_1$ norm, when we expect a sparse
solution to a regression problem,  is often justified by $\|w\|_1$ being the ``convex
envelope'' of $\|w\|_0$ (the number of non-zero coordinates
of a vector $w\in\R^d$).  That is, $\|w\|_1$ is the tightest convex lower
bound on $\|w\|_0$.  
But we must be careful with this
statement---
for sparse vectors with large entries,
$\|w\|_0$ can be small while $\|w\|_1$ is large.
In order to discuss convex lower bounds on $\|w\|_0$, we must
impose some scale constraint.  A more accurate statement is that
$\|w\|_1 \leq \|w\|_{\infty}\|w\|_0$,
and so, when the magnitudes of entries in $w$ are bounded by $1$, then
$\|w\|_1 \leq \|w\|_0$, and indeed it is the largest such
convex lower bound.  Viewed as a convex outer relaxation, 
\beq S^{(\infty)}_k:= \big\{ w
 \,\big|\, \|w\|_0 \leq k, \|w\|_{\infty} \leq 1 \big\}
\subseteq \big\{ w \,\big|\, \|w\|_1 \leq k \big\}\;.\eeq
 Intersecting the right-hand-side with the $\ell_{\infty}$ unit ball, we get
the tightest convex outer bound (convex hull) of $S^{(\infty)}_{k}$:
\beq \big\{ w   \,\big|\, \|w\|_1 \leq k, \|w\|_{\infty}\leq 1
\big\} = \conv (S^{(\infty)}_k)\;.\eeq

However, in our view, this relationship between $\|w\|_1$ and
$\|w\|_0$ yields disappointing learning guarantees, and does not
appropriately capture the success of the $\ell_1$ norm as a surrogate
for sparsity.  In particular, the sample complexity\footnote{
We define this as the number of observations needed in
order to ensure expected prediction error no more than $\epsilon$
worse than that of
the best $k$-sparse predictor, for an arbitrary constant $\epsilon$
(that is, we suppress the dependence on $\epsilon$ and focus on the
dependence on the sparsity $k$ and dimensionality $d$).} of
learning a linear predictor with $k$ non-zero entries by empirical
risk minimization inside this class (an NP-hard optimization problem)
scales as\footnote{This is based on bounding the VC-subgraph dimension of this
  class, which is essentially the effective number of parameters.}
$O(k \log d)$, but relaxing to the constraint $\|w\|_1\leq k$ yields a
sample complexity which scales as $O(k^2 \log d)$, because the sample
complexity of $\ell_1$-regularized learning scales quadratically with
the $\ell_1$ norm \cite{kakade,zhang2002covering}.

Perhaps a better reason for the $\ell_1$ norm being a good surrogate for sparsity
is that, not only do we expect the magnitude of each entry of $w$ to
be bounded, but we further expect $\|w\|_2$ to be small.  In a
regression setting, with a vector of features $x$,
 this can be justified when $E[(x^{\top}w)^2]$ is
bounded (a reasonable assumption) and the features are not too
correlated---see, e.g. \cite{srebro_smooth}.  More broadly, especially in the presence
of correlations, we might require this as a modeling assumption to aid
in robustness and generalization.  In any case, we have
$\|w\|_1 \leq \|w\|_2 \sqrt{\|w\|_0}$,
and so if we are interested in predictors with bounded $\ell_2$ norm,
 we can motivate the $\ell_1$ norm through the following
relaxation of sparsity, where the scale is now set by the $\ell_2$
norm:
\begin{equation}
\big\{ w
  \,\big|\, \|w\|_0 \leq k, \|w\|_2 \leq B \big\}
\subseteq \big\{ w    \,\big|\, \|w\|_1 \leq B \sqrt{k} \big\}\;.
\end{equation}
The sample complexity when using the relaxation now scales
as\footnote{More precisely, the sample complexity is $O(B^2 k \log d)$, where the
  dependence on $B^2$ is to be expected.  Note that if feature vectors
  are $\ell_{\infty}$-bounded (i.e.~individual features are bounded),
  the sample complexity when using only $\|w\|_2\leq B$ (without a
  sparsity or $\ell_1$ constraint) scales as $O(B^2 d)$.  That is,
  even after identifying the correct support, we still need a sample
  complexity that scales with $B^2$.} $O(k \log d)$.

\paragraph{Sparse + $\ell_2$ constraint.}  Our starting point is then
that of combining sparsity and $\ell_2$ regularization, and learning a
sparse predictor with small $\ell_2$ norm.  We are thus interested in
classes of the form
\begin{equation}
  \label{eq:02vs1}
  S^{(2)}_k:=\left\{ w
     \,\big|\,  \|w\|_0 \leq k, \|w\|_2 \leq 1 \right\}\;.
\end{equation}

As discussed above, the class $\{ \|w\|_1 \leq \sqrt{k} \}$ (corresponding to
the standard Lasso) provides a convex relaxation of $S^{(2)}_k$.
But it is clear that we can get a tighter convex relaxation by
keeping the $\ell_2$ constraint as well:
\begin{equation}
  \label{eq:02vsENvs1}
S^{(2)}_k
\subseteq \left\{ w  \,\big|\, \|w\|_1 \leq\sqrt{k},
  \|w\|_2\leq 1 \right\}
\subsetneq \left\{ w  \,\big|\, \|w\|_1 \leq \sqrt{k} \right\}\;.
\end{equation}
Constraining (or equivalently, penalizing) both the $\ell_1$ and
$\ell_2$ norms, as in \eqref{eq:02vsENvs1}, is known as the ``elastic
net'' \cite{elastic_demol,elastic} and has indeed been advocated as a better
alternative to the Lasso.  In this paper, we ask
whether the elastic net is the {\em tightest} convex relaxation to sparsity
plus $\ell_2$ (that is, to $S^{(2)}_k$) or whether a
tighter, and better, convex relaxation is possible.

\paragraph{A new norm.}
We consider the convex hull (tightest convex outer bound) of $S^{(2)}_k$,
\begin{equation}
  \label{eq:conv02}
 C_k:=\conv(S^{(2)}_k)= \conv \left\{ w \,\big|\,  \|w\|_0 \leq k, \|w\|_2 \leq 1 \right\}\;.
\end{equation}
We study the gauge function associated with this convex set, that is, the
norm whose unit ball is given by \eqref{eq:conv02}, which we call the
{\em $k$-support norm}.  We show that, for $k>1$, this is indeed a tighter
convex relaxation than the elastic net (that is, both inequalities in
\eqref{eq:02vsENvs1} are in fact strict inequalities), and is therefore
a better convex constraint than the elastic net when seeking a
sparse, low $\ell_2$-norm linear predictor.  We thus advocate using it
as a replacement for the elastic net.

However, we also show that the gap between the elastic net and the
$k$-support norm is at most a factor of $\sqrt{2}$, corresponding to a
factor of two difference in the sample complexity. Thus, our work can
also be interpreted as justifying the use of the elastic net, viewing it as
a fairly good approximation to the tightest possible convex relaxation
of sparsity intersected with an $\ell_2$ constraint.  Still, even a factor of two
should not necessarily be ignored and, as we show in our experiments,
using the tighter $k$-support norm can indeed be beneficial.

To better understand the $k$-support norm, we show in Section
\ref{sec:def} that it can also be described as the group lasso with
overlaps norm \cite{jacob2009group} corresponding to all
$\binom{d}{k}$ subsets of $k$ features. Despite the exponential number
of groups in this description, we show that the $k$-support norm can
be calculated efficiently in time $O(d \log d)$ and that its dual is
given simply by the $\ell_2$ norm of the $k$ largest entries.  We also provide efficient first-order optimization algorithms for
learning with the $k$-support norm.



\paragraph{Related Work}


In many learning problems of interest, Lasso has
been observed to shrink too many of the variables of $w$ to
zero. In particular, in many applications, when a group of variables is highly correlated,
the Lasso may prefer a sparse solution, but we might gain more predictive
accuracy by including all the correlated variables in our model.
These drawbacks have recently motivated the use of various other regularization methods,
such as the {\em elastic net} \cite{elastic}, 
which penalizes the regression coefficients $w$ with
a combination of $\ell_1$ and $\ell_2$ norms:
\beq
\min \left\{ \frac{1}{2} \|Xw-y\|^2 + \lambda_1 \, \|w\|_1 + 
\lambda_2 \, \|w\|_2^2 : w \in \R^d \right\} \,,
\label{eq:def_elastic}
\eeq
where for a sample of size $n$, $y\in\R^n$ is the vector of response values,
and $X\in\R^{n\times d}$ is a matrix with column $j$ containing the values of 
feature $j$.

The elastic net can  be viewed as a trade-off between $\ell_1$ regularization (the Lasso)
and $\ell_2$ regularization (Ridge regression \cite{hoerl1970ridge}),
depending on the relative values of $\lambda_1$ and $\lambda_2$. In particular, when $\lambda_2=0$,
\eqref{eq:def_elastic} is equivalent to the Lasso.
This method, and the other methods discussed below, 
 have  been observed to significantly outperform Lasso in many real applications.

The pairwise elastic net (PEN), proposed by \cite{pen}, has
a penalty function that accounts for similarity among features:
\beq
\|w\|^{PEN}_{R}= \|w\|^2_2 + \|w\|^2_1 -|w|^{\top}R|w|\;,\eeq
where $R\in[0,1]^{p\times p}$ is a matrix with $R_{jk}$ 
measuring similarity between features $X_j$ and $X_k$.
The trace Lasso \cite{trace_lasso} is a second method proposed to handle correlations 
within $X$,  defined by
\beq
\|w\|^{trace}_X=\|X\text{diag}(w)\|_*\;,\eeq
where $\|\cdot\|_*$ denotes the matrix trace-norm (the sum of the singular
values) and promotes a low-rank solution. 
If the features are orthogonal, then both the PEN and the Trace Lasso are
equivalent to the Lasso. If the features are all identical, then both penalties
are equivalent to Ridge regression (penalizing $\|w\|_2$).
Another existing penalty is OSCAR \cite{oscar}, given by
\beq
\|w\|^{OSCAR}_c=\|w\|_1+c\sum_{j<k}\max\{|w_j|,|w_k|\}\;.\eeq
Like the elastic net, each one of these three methods also ``prefers''
averaging similar features over selecting a single feature.


\section{The $k$-Support Norm}
\label{sec:def}

One argument for the elastic net has been the flexibility of tuning 
the cardinality $k$ of the regression vector $w$. Thus, when groups
of correlated variables are present, a larger $k$ may be learned, 
which corresponds to a higher $\lambda_2$ in \eqref{eq:def_elastic}.
A more natural way to obtain such an effect of tuning the cardinality
is to consider the convex hull of cardinality $k$ vectors,
\beq
C_k = \conv(S^{(2)}_k)= \conv \{ w\in\R^d \,\big|\, \|w\|_0\leq k, \|w\|_2 \leq 1 \}.
\eeq
Clearly the sets $C_k$ are nested, and $C_1$ and $C_d$ are the unit balls
for the $\ell_1$ and $\ell_2$ norms, respectively.
Consequently we define the {\em $k$-support norm} as the norm whose
unit ball equals $C_k$ (the gauge function associated with the $C_k$ ball).\footnote{
The gauge function $\gamma_{C_k}: \R^d\to\R\cup\{+\infty\}$ is defined as
$\gamma_{C_k}(x) = \inf\{\lambda \in \R_+ : x \in \lambda C_k\}$.}
An equivalent definition is the following variational formula:

\begin{definition}
Let $k\in\{1,\dots,d\}$. The {\em $k$-support norm} $\|\cdot\|^{sp}_k$
is defined, for every $w\in\R^d$, as  
\beq\|w\|^{sp}_k := \min \Biggl\{\! \sum\limits_{I \in \calG_k} \|v_I\|_2
: \supp(v_I) \subseteq I, 
\sum\limits_{I \in \calG_k} v_I = w \!\Biggr\} \;,
\eeq
where $\calG_k$ denotes the set of all subsets of 
$\{1,\dots,d\}$ of cardinality at most $k$.
\label{def:overlap}
\end{definition}
The equivalence is immediate by rewriting $v_I=\mu_I z_I$ in the above
definition, where $\mu_I \geq 0, z_I \in C_k, \forall I\in\calG_k$,
$\sum_{I \in \calG_k}\mu_I = 1$. In addition, this immediately implies
that $\|\cdot\|^{sp}_k$ is indeed a norm. In fact, the $k$-support
norm is equivalent to the norm used by the group lasso with overlaps
\cite{jacob2009group}, when the set of overlapping groups is chosen to
be $\calG_k$ (however, the group lasso has traditionally been used for
applications with some specific known group structure, unlike the case
considered here).

Although the variational definition \ref{def:overlap} is not amenable to computation because of the 
exponential growth of the set of groups $\calG_k$, the $k$-support norm
is computationally very tractable, with an $O(d\log d)$ algorithm described 
in Section \ref{sec:computation}.

As already mentioned, $\|\cdot\|_1^{sp} = \|\cdot\|_1$ and
$\|\cdot\|_d^{sp} = \|\cdot\|_2$.  The unit ball of this new norm in
$\R^3$ for $k=2$ is depicted in Figure \ref{fig:ball}. We immediately
notice several differences between this unit ball and the elastic net
unit ball. For example, at points with cardinality $k$ and $\ell_2$
norm equal to $1$, the $k$-support norm is not differentiable, but
unlike the $\ell_1$ or elastic-net norm, it {\em is} differentiable at
points with cardinality less than $k$. Thus, the $k$-support norm is
less ``biased'' towards sparse vectors than the elastic net and the
$\ell_1$ norm.
\begin{figure}[t]
\begin{center}
\includegraphics[width=0.2\textwidth]{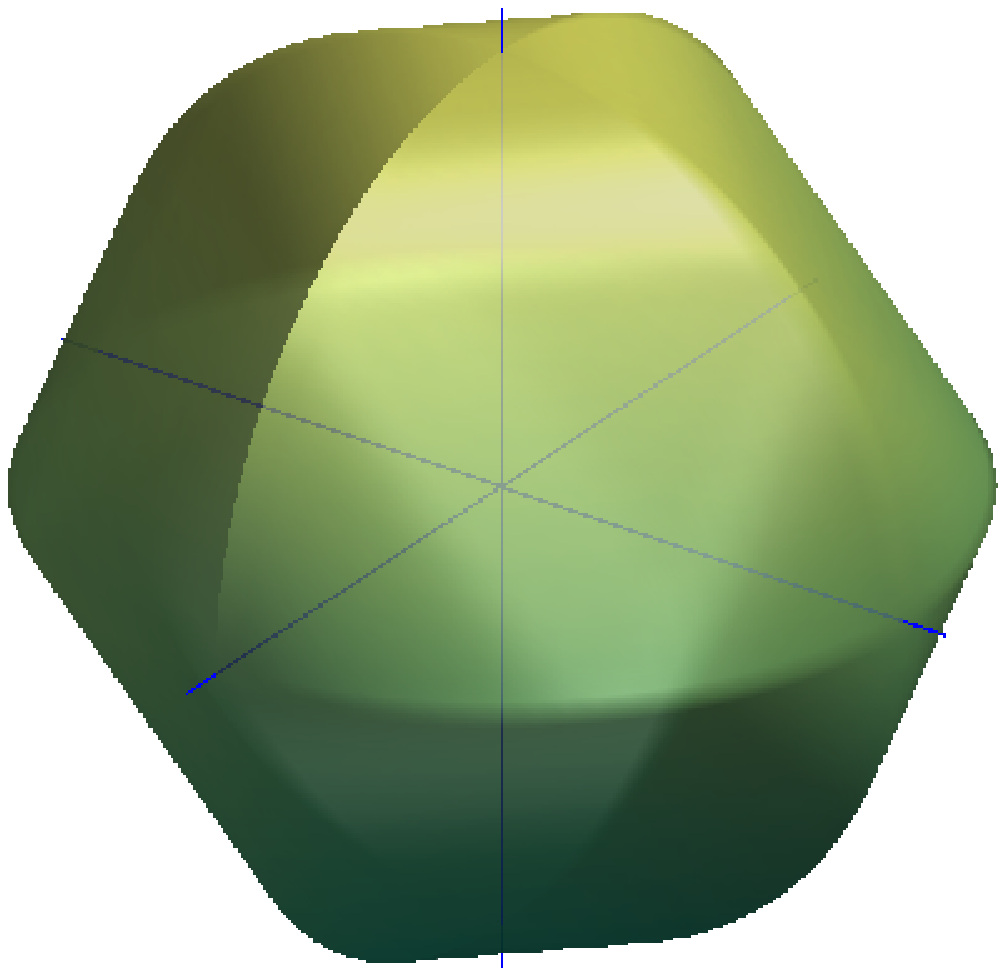}
\hspace{0.2\textwidth}
\includegraphics[width=0.18\textwidth]{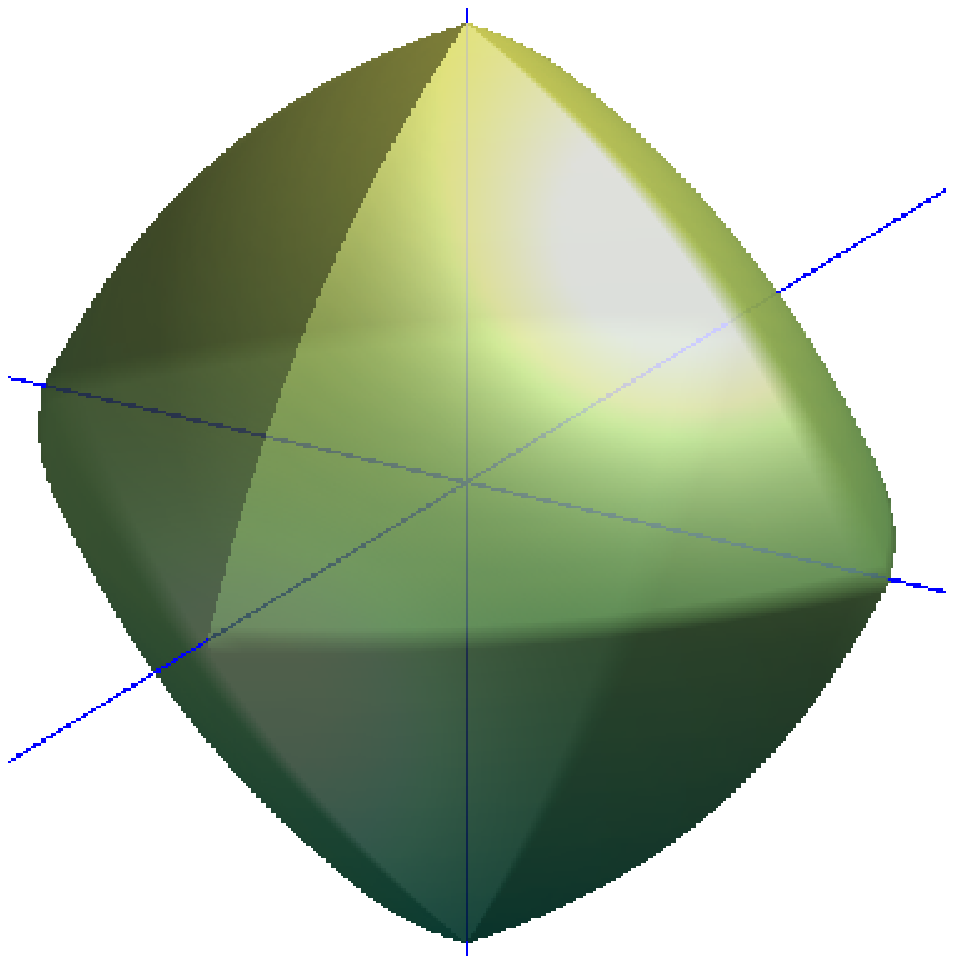}
\end{center}
\vspace{-0.2in}
\small\caption{\small Unit ball of the $2$-support norm (left) 
and of the elastic net (right) on $\R^3$.}
\label{fig:ball}
\end{figure}



\subsection{The Dual Norm}

It is interesting and useful to compute the dual of the $k$-support norm.
We follow the notation of \cite{bhatia} for ordered vectors:
 for any $w\in\R^d$, $|w|$ is
the vector of absolute values,
and $w_i\down$ is the $i$-th largest element of $w$.
We have
 \begin{align}
    \|u\|^{sp^*}_k  = \max \left\{\langle w , u \rangle :
   \|w\|^{sp}_k \leq 1 \right\} & =
 \max \left\{ \left( \sum\limits_{i\in I} u_{i}^2 \right)^\frac{1}{2}
: I \in \calG_k \right\} \label{eq:max} \\
& = 
\left( \sum\limits_{i=1}^k (|u|^\downarrow_i)^2 \right)^\frac{1}{2}
=: \|u\|_{(k)}^{(2)} \;.
\end{align}
This is the $\ell_2$-norm of the largest $k$ entries in
$u$, and is known as the $2$-$k$ {\em symmetric gauge norm}
\cite{bhatia}.

Not surprisingly, this dual norm interpolates between the $\ell_2$ norm
(when $k=d$ and all entries are taken) and the $\ell_{\infty}$ norm
(when $k=1$ and only the largest entry is taken).  This parallels the
interpolation of the $k$-support norm between the $\ell_1$ and
$\ell_2$ norms.

Like the $\ell_p$ norms and elastic net, the $k$-support norm and its
dual are {\em symmetric gauge functions}, that is, sign- and
permutation-invariant norms. For properties of such norms, see
\cite{bhatia}.

\subsection{Computation of the Norm}
\label{sec:computation}

In this section, we derive an alternative formula for the
$k$-support norm, which leads to computation of the value of the norm
in $O(d \log d)$ steps. 

\begin{proposition}\label{prop:norm}
For every $w\in\R^d$,
\beq
\|w\|_k^{sp} = \left( \sum\limits_{i=1}^{k-r-1} (|w|\down_i)^2 + 
\frac{1}{r+1} \left(\sum\limits_{i=k-r}^d |w|_i\down\right)^2
\right)^\frac{1}{2}\;,\eeq
where, letting $|w|\down_0$ denote $+\infty$, $r$ is the unique integer in $\{0,\dots,k-1\}$
satisfying
\beq
|w|\down_{k-r-1} ~ > ~ 
\frac{1}{r+1}\sum\limits_{i=k-r}^d |w|_i\down 
~\geq ~ |w|\down_{k-r} \;.
\label{eq:cond}
\eeq 
\end{proposition}

This result shows that $\|\cdot\|^{sp}_k$ trades off between the $\ell_1$ and
$\ell_2$ norms in a way that favors sparse vectors but allows for cardinality larger than $k$.
It combines the uniform shrinkage of an $\ell_2$ penalty 
for the largest components, with the sparse shrinkage of an $\ell_1$ penalty
for the smallest components.

\begin{proof}[{\bf Proof of Proposition \ref{prop:norm}}]
We will use the inequality 
$\lb w, u\rb \leq \lb w\down, u\down\rb$
 \cite{HLP}. We have
\begin{align}
\frac{1}{2}(\|w\|^{sp}_k)^2 
& = \max \left \{ \lb u,w\rb - \frac{1}{2} (\|u\|^{(2)}_{(k)})^2 :  u\in \R^d\right\} 
\\& = \max \Biggl\{ \sum_{i=1}^{d} \alpha_i |w|\down_i -
\frac{1}{2} \sum\limits_{i=1}^k \alpha_i^2 
: 
 \alpha_1\geq \dots\geq \alpha_d\geq 0 \Biggr\} 
 \\ & = \max \Biggl \{ \sum_{i=1}^{k-1} \alpha_i |w|\down_i 
+ \alpha_k \sum_{i=k}^{d}  |w|\down_i
-\frac{1}{2} \sum\limits_{i=1}^k \alpha_i^2 
:  \alpha_1\geq \dots\geq \alpha_k\geq 0 \Biggr\} \,.
\end{align}
Let $A_r : = \sum\limits_{i=k-r}^d |w|_i\down$
for $r\in\{0,\dots,k-1\}$.
If $A_0 < |w|_{k-1}\down$ then the solution $\alpha$ is given by
$\alpha_i = |w|_i\down$ for $i=1,\dots,(k-1)$, $\alpha_i = A_0$ for $i=k,\dots,d$.
If $A_0 \geq |w|_{k-1}\down$ then the optimal $\alpha_k$, $\alpha_{k-1}$
lie between $|w|_{k-1}\down$ and $A_0$, and have to be equal.
So, the maximization becomes
\beq
\max \Biggl \{ \sum_{i=1}^{k-2} \alpha_i |w|\down_i 
-\frac{1}{2} \sum\limits_{i=1}^{k-2} \alpha_i^2
+ A_1 \alpha_{k-1} -\alpha_{k-1}^2 
: 
\alpha_1\geq \dots\geq \alpha_{k-1}\geq 0 \Biggr\} \,.
\eeq
If $A_0 \geq |w|_{k-1}\down$ and $|w|_{k-2}\down > \frac{A_1}{2}$ then
the solution is
$\alpha_i = |w|_i\down$ for $i=1,\dots,(k-2)$, $\alpha_i = \frac{A_1}{2}$ for $i=(k-1),\dots,d$.
Otherwise we proceed as before and continue this process.
At stage $r$ the process terminates if
$A_0 \geq |w|_{k-1}\down, \dots, \frac{A_{r-1}}{r} \geq |w|_{k-r}\down$, $
\frac{A_r}{r+1} < |w|_{k-r-1}\down$
and all but the last two inequalities are redundant. Hence the condition
can be rewritten as \eqref{eq:cond}. One optimal solution is 
$\alpha_i = |w|_i\down$ for $i=1,\dots,k-r-1$, $\alpha_i = \frac{A_r}{r+1}$ for $i=k-r,\dots,d$.
This proves the claim.
\end{proof}

\subsection{Learning with the $k$-support norm}

We thus propose using learning rules with $k$-support norm
regularization.  These are appropriate when we would like to learn a
sparse predictor that also has low $\ell_2$ norm, and are especially
relevant when features might be correlated (that is, in almost all
learning tasks) but the correlation structure is not known in
advance.  For regression problems with squared error loss, the
resulting learning rule is of the form
\beq \min
\left\{ \frac{1}{2} \|Xw-y\|^2 + \frac{\lambda}{2} \,
  \left(\|w\|^{sp}_k\right)^2 : w \in \R^d \right\}
\label{eq:support}
\eeq
with $\lambda > 0$ a regularization parameter and $k\in\{1,\dots,d\}$ 
also a parameter to be tuned. As typical in regularization-based methods,
both $\lambda$ and $k$ can be selected by cross validation \cite{hastie}.
Although we have motivated this norm by considering $S_k^{(2)}$, the set of
$k$-sparse unit vectors, the parameter $k$ does not necessarily correspond
to the sparsity level of the fitted vector of coefficients,
and should be chosen via cross-validation independently of the desired sparsity level.

\section{Relation to the Elastic Net}
\label{sec:elastic}

Recall that the elastic net with penalty parameters $\lambda_1$ and $\lambda_2$
selects a vector of coefficients given by
\beq
\arg\min \left\{ \frac{1}{2} \|Xw-y\|^2 + \lambda_1 \, \|w\|_1 + 
\lambda_2 \, \|w\|_2^2 \right\} \,. 
\label{eq:elastic}
\eeq

For ease of comparison with the $k$-support  norm, we first show that
the set of optimal solutions for the elastic net, when the parameters are varied, is the same
as for the norm
\beq
\|w\|^{el}_k := \max \left\{\|w\|_2, \frac{\|w\|_1}{\sqrt{k}} \right\} \;,
\label{eq:el_norm}
\eeq when $k\in[1,d]$, corresponding to the unit ball in
\eqref{eq:02vsENvs1} (note that $k$ is not necessarily an integer). To see
this, let $\hat{w}$ be a solution to (\ref{eq:elastic}), and let
$k:=({\|\hat{w}\|_1}/{\|\hat{w}\|_2})^2\in[1,d]\;$.

Then for any $w\neq \hat{w}$, if $\|w\|^{el}_k\leq
\|\hat{w}\|^{el}_k$, then $\|w\|_p\leq \|\hat{w}\|_p$ for $p=1,2$.
Since $\hat{w}$ is a solution to \eqref{eq:elastic}, therefore,
$\|Xw-y\|^2_2\geq \|X\hat{w}-y\|^2_2$. This proves that, for some
constraint parameter $B$, \beq
\hat{w}=\arg\min\left\{\frac{1}{n}\|Xw-y\|^2_2 \, : \,
  \|w\|^{el}_k\leq B\right\}\;. \eeq

Like the $k$-support norm, the elastic net interpolates between the $\ell_1$ and $\ell_2$ norms. In fact, when $k$ is an integer, any $k$-sparse unit vector $w\in\R^d$ must lie in the unit ball of $\|\cdot \|^{el}_k$. Since the $k$-support norm gives the convex hull of all $k$-sparse unit vectors, this immediately implies that
\beq
\|w\|^{el}_k\leq \|w\|^{sp}_k \quad\quad \forall \ w\in\R^d\;.
\eeq
The two norms are not equal, however.
The difference between the two is illustrated in Figure
\ref{fig:ball}, where we see that the $k$-support norm is more
``rounded''. 

To see an example where the two norms are not equal, we set $d=1+k^2$ for some
large $k$, and let
$w=(k^{1.5},1,1,\dots,1)^{\top}\in\R^d$.
Then
\beq\|w\|^{el}_k=\max\left\{\sqrt{k^3+k^2},\frac{k^{1.5}+k^2}{\sqrt{k}}\right\}=k^{1.5}\left(1+\frac{1}{\sqrt{k}}\right).\eeq
Taking $u=(\frac{1}{\sqrt{2}}, \frac{1}{\sqrt{2k}},\frac{1}{\sqrt{2k}}, \dots, \frac{1}{\sqrt{2k}})^{\top}$, we have 
$\|u\|^{(2)}_{(k)}< 1$,
and recalling this norm is dual to the $k$-support norm:
\beq\|w\|^{sp}_k> \langle w,u\rangle =\frac{k^{1.5}}{\sqrt{2}}+k^2\cdot\frac{1}{\sqrt{2k}}=\sqrt{2}\cdot k^{1.5}\;.\eeq
In this example, we see that the two norms can differ by as much as 
a factor of $\sqrt{2}$.
We now show that this is actually the most by which they can differ.
\begin{proposition}
$\|\cdot\|^{el}_k \leq \|\cdot\|_k^{sp} < \sqrt{2} \, \|\cdot\|^{el}_k$.
\label{prop:bound}
\end{proposition}

\begin{proof}
We show that these bounds hold in the duals of the two norms. 
First, since $\|\cdot\|^{el}_k$ is a maximum over the $\ell_1$ and $\ell_2$ norms,
its dual is given by
\beq
\|u\|^{(el)^*}_k:=\inf_{a\in\R^d} \left\{\|a\|_2 + \sqrt{k}\cdot \|u-a\|_{\infty}\right\}\;
\eeq
Now take any $u\in\R^d$. First we show
$\|u\|^{(2)}_{(k)}\leq \|u\|^{(el)^*}_k$.
Without loss of generality, we take $u_1\geq \dots \geq u_d\geq 0$. 
For any $a\in \R^d$,
\beq
\|u\|_{(k)}^{(2)} =\|u_{1:k}\|_2 \leq \|a_{1:k}\|_2 + \|u_{1:k} - a_{1:k}\|_2
\leq \|a\|_2 + \sqrt{k} \|u-a\|_\infty \,. 
\eeq
Finally, we show that $ \|u\|^{(el)^*}_k< \sqrt{2}\,\|u\|^{(2)}_{(k)}$.
Taking
$a=(u_1-u_{k+1},\dots,u_k-u_{k+1},0,\dots,0)^{\top}$,
we have
\beq\begin{split}
\|u\|^{(el)^*}_k&\leq \|a\|_2+\sqrt{k}\cdot\|u-a\|_{\infty}
=\sqrt{\sum_{i=1}^k (u_i-u_{k+1})^2} + \sqrt{k}|u_{k+1}| \\
&\leq \sqrt{\sum_{i=1}^k (u_i^2-u_{k+1}^2)} + \sqrt{k\, u_{k+1}^2}
\leq \sqrt{2}\cdot \sqrt{\sum_{i=1}^k (u_i^2-u_{k+1}^2)  + k\, u_{k+1}^2}
\\ & =\sqrt{2}\, \|u\|^{(2)}_{(k)}\;.
\end{split}\eeq

Furthermore, this yields a strict inequality, because if $u_1>u_{k+1}$,
the next-to-last inequality is strict, while if $u_1=\dots=u_{k+1}$, then
the last inequality is strict.
\end{proof}


\section{Optimization}
\label{sec:prox}

Solving the optimization problem \eqref{eq:support} efficiently
can be done with a first-order proximal algorithm. 
Proximal methods -- see
\cite{fista,combettes,Nesterov07,tseng08,tseng10} and references
therein -- are used to solve composite problems of the form
$\min\{f(x) + \omega(x): x\in\R^d\}$, where the loss function $f(x)$
and the regularizer $\omega(x)$ 
are convex functions, and $f$ is smooth
with an $L$-Lipschitz gradient.
These methods require fast computation of the gradient $\nabla f$
and the {\em proximity operator}
\beq \prox_\omega(x)
:= \argmin 
\left\{ \dfrac{1}{2} \|u-x\|^2 + \reg(u) : u \in \rd \right\}\;.\eeq
In particular, {\em accelerated first-order methods}, proposed by Nesterov
\cite{nesterov2005smooth,Nesterov07} require two levels of memory 
at each iteration and exhibit an optimal  $O\left(\frac{1}{T^2}\right)$ 
convergence rate for the objective after $T$ iterations.

To obtain a proximal method for $k$-support regularization,
it suffices to compute the proximity map of 
$g = \frac{1}{2L}(\|\cdot\|^{sp}_k)^2$, for any $L>0$.
This can be done in $O(d(k+\log d))$ steps
with Algorithm \ref{alg:prox}.

\begin{algorithm}[th]
\caption{Computation of the proximity operator.}
\begin{algorithmic}
\STATE {\bf Input} $v \in\rd$ 
\STATE {\bf Output} $q = \prox_{\frac{1}{2L}(\|\cdot\|^{sp}_k)^2}(v)$
\STATE Find $r\in\{0,\dots,k-1\}$, $\ell \in \{k, \dots, d\}$
such that 
\beq
\tfrac{1}{L+1} z_{k-r-1} > \tfrac{ T_{r,\ell}}{\ell-k+(L+1)r+L+1} 
\geq \tfrac{1}{L+1}z_{k-r}
\label{eq:prox_cond1}
\eeq
\beq
z_{\ell} > \tfrac{T_{r,\ell}}{\ell-k+(L+1)r+L+1} \geq z_{\ell+1}
\label{eq:prox_cond2}
\eeq
where $z := |v|\down$, $z_0 := +\infty$, $z_{d+1} := -\infty$,~ 
$T_{r,\ell} := \sum\limits_{i=k-r}^\ell z_i$ 
\STATE $q_i \leftarrow \begin{cases}
\tfrac{L}{L+1}z_i & \text{if}~ { i=1,\dots,k-r-1} \\
z_i - \tfrac{T_{r,\ell}}{\ell-k+(L+1)r+L+1}
& \text{if}~ {i=k-r,\dots,\ell} \\
0 & \text{if}~  { i=\ell+1,\dots,d}
\end{cases}$
\STATE Reorder and change signs of $q$ to conform with $v$ 
\end{algorithmic}
\label{alg:prox}
\end{algorithm}

\begin{algorithm}[t]
\caption{Accelerated $k$-support regularization.}
\begin{algorithmic}
\STATE $w_1=\alpha_1 \in\rd$, $\theta_1 \leftarrow 1$
\FOR {t=1,2,\dots}
\STATE $\theta_{t+1} \leftarrow \tfrac{1+\sqrt{1+4\theta_t^2}}{2}$
\STATE $w_{t+1} \leftarrow \prox_{\frac{\lambda}{2L}(\|\cdot\|^{sp}_k)^2} 
\left(\alpha_t - \frac{1}{L}X\trans(X\alpha_t-y)\right)$
using Algorithm \ref{alg:prox}
\STATE $\alpha_{t+1} \leftarrow  w_{t+1} +
\tfrac{\theta_t-1}{\theta_{t+1}} (w_{t+1}-w_t)$
\ENDFOR 
\end{algorithmic}
\label{alg:nest}
\end{algorithm}

\begin{proof}[{\bf Proof of Correctness of Algorithm \ref{alg:prox}}]
  Since the support-norm is sign and permutation invariant, $\prox(v)$
  has the same ordering and signs as $v$. 
  Hence, without loss of generality, we may assume that $v_1\geq
  \dots\geq v_d\geq 0$ and require that $q_1\geq \dots\geq q_d\geq 0$,
  which follows from inequality \eqref{eq:prox_cond1} and the
  fact that $z$ is ordered.

Now, $q=\prox(v)$ is equivalent to $Lz-Lq = Lv-Lq \in \partial 
\frac{1}{2}(\|\cdot\|^{sp}_k)^2(q)$. It suffices to show
that, for $w=q$, $Lz-Lq$ is an optimal $\alpha$ in the proof of Proposition \ref{prop:norm}.
Indeed, $A_r$ corresponds to $\sum\limits_{i=k-r}^d q_i = 
\sum\limits_{i=k-r}^\ell 
\left(z_i - \tfrac{T_{r,\ell}}{\ell-k+(L+1)r+L+1}\right) = 
 T_{r,\ell} - \tfrac{(\ell-k+r+1)T_{r,\ell}}{\ell-k+(L+1)r+L+1}
 = (r+1) \tfrac{L\,T_{r,\ell}}{\ell-k+(L+1)r+L+1}$ and
\eqref{eq:cond} is equivalent to condition \eqref{eq:prox_cond1}.
For $i\leq k-r-1$, we have $Lz_i-Lq_i = q_i$. For
 $k-r\leq i \leq \ell$, we have $Lz_i-Lq_i = \tfrac{1}{r+1}A_r$.
  For $i\geq \ell+1$, since $q_i=0$, we only need
     $Lz_i-Lq_i \leq \tfrac{1}{r+1}A_r$,
which is true by \eqref{eq:prox_cond2}.  
\end{proof}

We can now apply a standard accelerated proximal method, such as FISTA
\cite{fista}, to \eqref{eq:support}, at each iteration using the
gradient of the loss and performing a prox step using Algorithm
\ref{alg:prox}.  The FISTA guarantee ensures us that, with appropriate
step sizes, after $T$ such iterations, we have:
\begin{equation}
  \frac{1}{2} \|Xw_T-y\|^2 + \frac{\lambda}{2} \, \left(\|w_T\|^{sp}_k\right)^2
  \leq  \Biggl(\frac{1}{2} \|X w^*-y\|^2 +  \frac{\lambda}{2} \, \left(\|w^*\|^{sp}_k\right)^2\Biggr)
  + \frac{2L\|w^*-w_1\|^2}{(T+1)^2}\;.
\label{eq:nest}
\end{equation}


\section{Empirical Comparisons}
\label{sec:exp}

Our theoretical analysis indicates that the $k$-support norm and the
elastic net differ by at most a factor of $\sqrt{2}$, corresponding to
at most a factor of two difference in their sample complexities and
generalization guarantees.  We thus do not expect huge differences
between their actual performances, but would still like to see whether
the tighter relaxation of the $k$-support norm does yield some gains.

\paragraph{Synthetic Data}

For the first simulation we follow \cite[Sec. 5, example 4]{elastic}.
In this experimental protocol, the target (oracle) vector equals 
\beq
w^*
= (\underbrace{3,
  \dots,3}_\text{15},\underbrace{0\dots,0}_\text{25})\;,\eeq with
$y=(w^*)\trans x + \calN(0,1)$.

The input data $X$ were generated from a normal distribution
such that components $1,\dots,5$ have the same random mean $Z_1\sim\calN(0,1)$,
components $6,\dots,10$ have mean $Z_2\sim\calN(0,1)$ and components $11,\dots,15$ have mean $Z_3\sim\calN(0,1)$.  
A total of $50$ data sets were created in this way, each containing
50 training points, 50 validation points and 350 test points.  
The goal is to achieve good prediction performance on the test data.

We compared the $k$-support norm with Lasso and the elastic net. 
We considered the ranges $k=\{1,\dots,d\}$ for $k$-support norm regularization, 
$\lambda = 10^i$, $i=\{-15,\dots,5\},$ for the regularization parameter
of Lasso and $k$-support regularization and the same range for the $\lambda_1,\lambda_2$ 
of the elastic net. For each method, the optimal set of parameters
was selected based on mean squared error on the validation set.
The error reported in Table \ref{tab:mse}
is the mean squared error with respect to the oracle $w^*$, namely 
$MSE = (\hat{w}-w^*)\trans V (\hat{w}-w^*) $,
where $V$ is the population covariance matrix of $X_{test}$.

Beyond the predictive gains, to further illustrate the effect of the
$k$-support norm, in Figure \ref{fig:coeffs} we show the coefficients
learned by each method, in absolute value.  For each image, one row
corresponds to the $w$ learned for one of the $50$ data sets. Whereas
the elastic net can learn higher values at the relevant features, a
better feature pattern with less variability emerges when using the
$k$-support norm.

\begin{figure}[t]
\begin{center}
\includegraphics[width=0.2\textwidth]{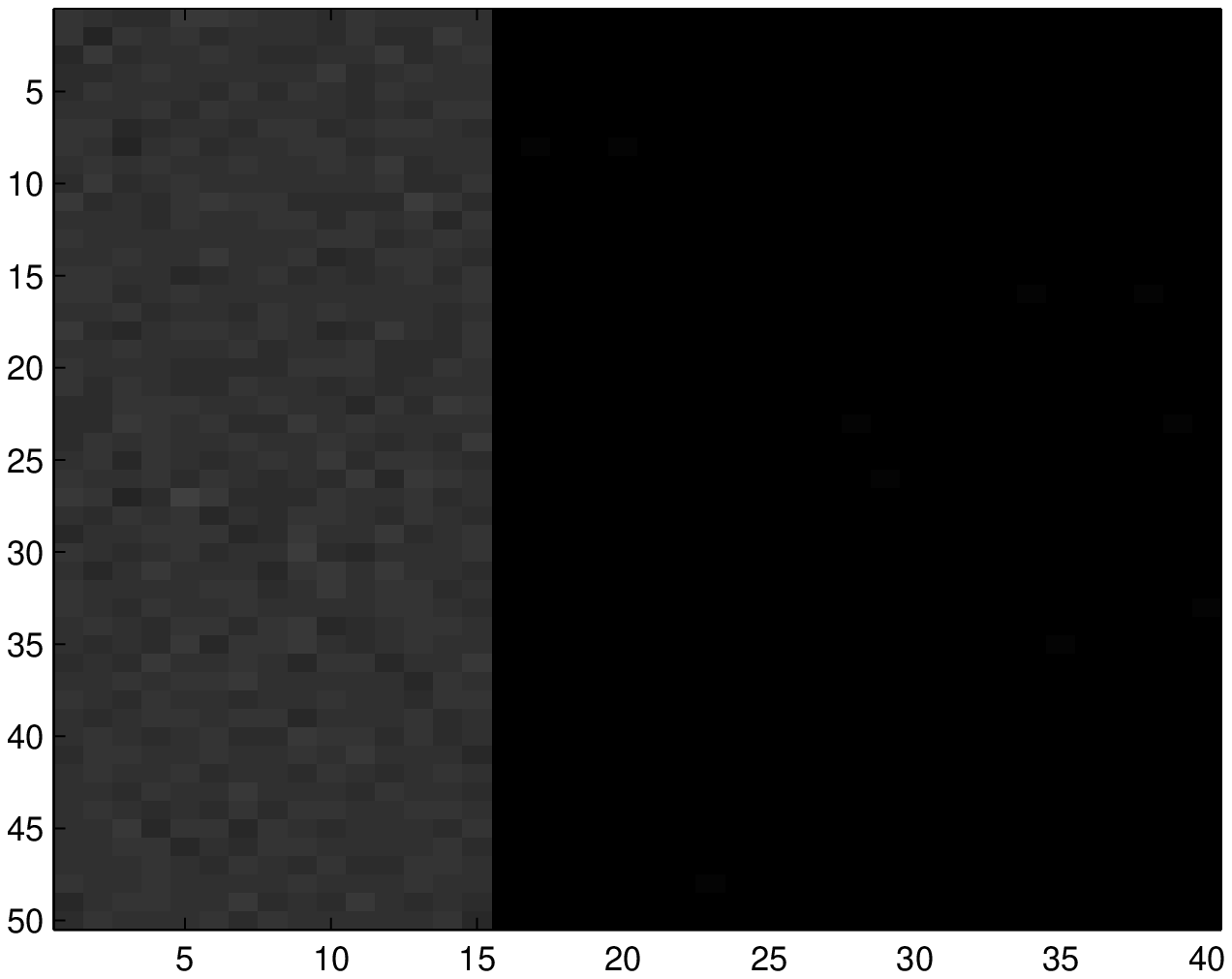}
\includegraphics[width=0.2\textwidth]{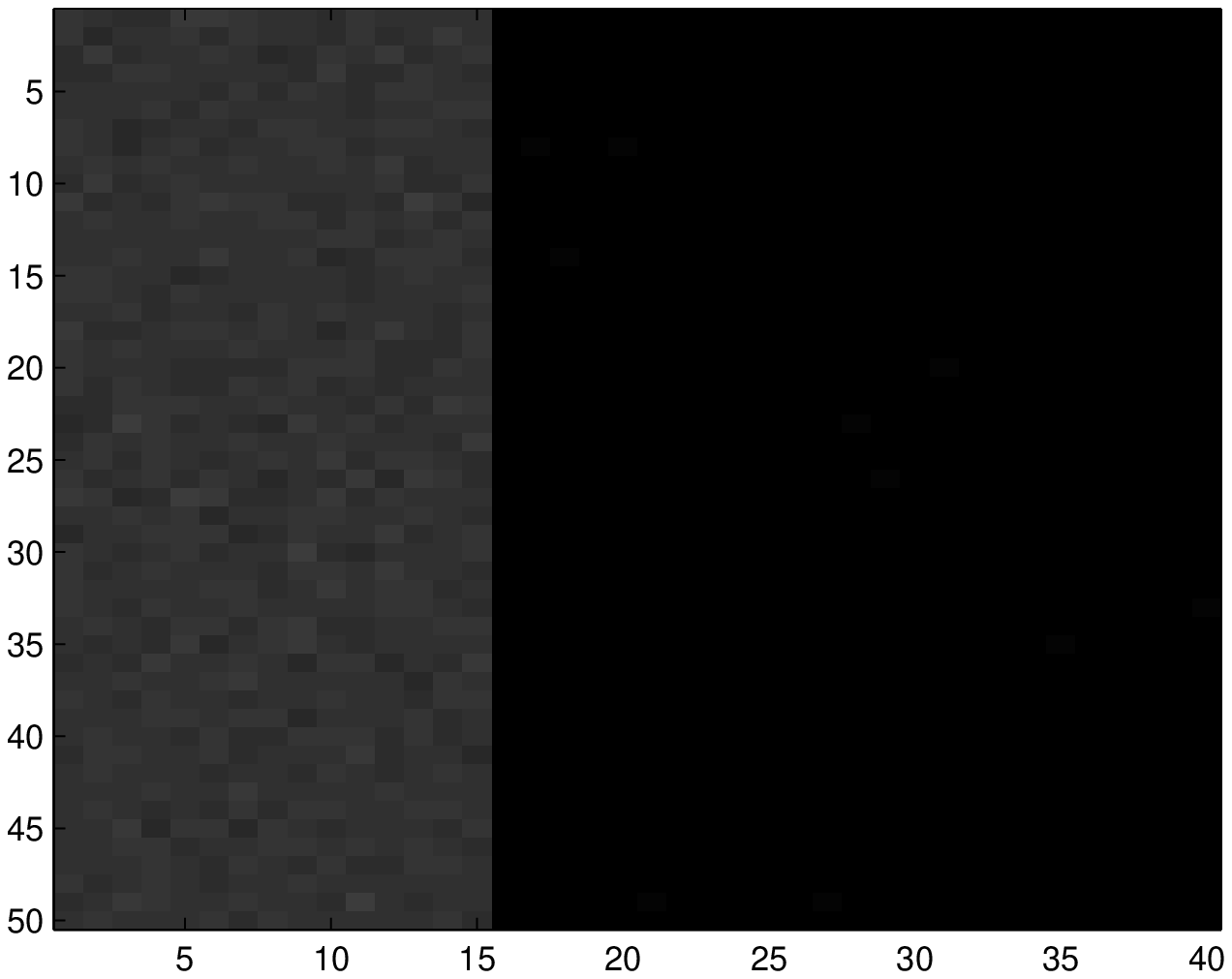}
\includegraphics[width=0.2\textwidth]{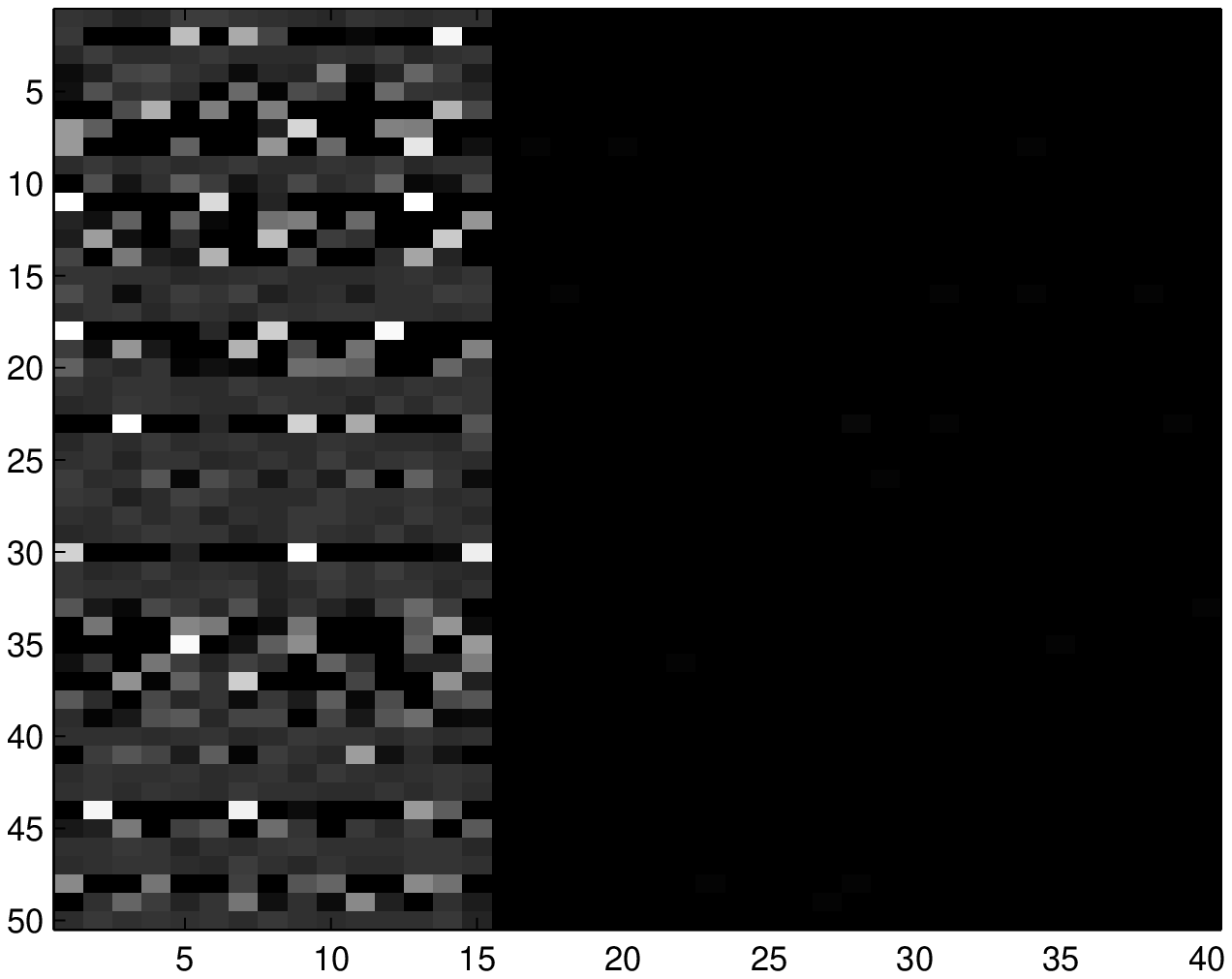}
\end{center}
{\small \caption{\small Solutions learned by each method for all the
    simulation data sets. Left to right: $k$-support, Lasso and elastic net.}}
\label{fig:coeffs}
\end{figure}

\paragraph{South African Heart Data}

This is a classification task which has been used in \cite{hastie}. There are 9 variables
and 462 examples, and the response is presence/absence of coronary heart disease.  
We normalized the data so that each predictor
variable has zero mean and unit variance.  We
then split the data 50 times randomly into training, validation, and test sets of
sizes 400, 30, and 32 respectively. For each method, parameters 
were selected using the validation data. In Tables \ref{tab:mse}, 
we report the MSE and accuracy of each method on the test data.
We observe that all three methods
have identical performance. 




\paragraph{20 Newsgroups}

This is a binary classification version of 20 newsgroups created in
\cite{keerthi_coste} which can be found in the LIBSVM data
repository.\footnote{\texttt{http://www.csie.ntu.edu.tw/$\sim$cjlin/libsvmtools/datasets/}}
The positive class consists of the 10 groups with names of form sci.*,
comp.*, or misc.forsale and the negative class consists of the other
10 groups.  To reduce the number of features, we removed the words
which appear in less than $3$ documents. We randomly split the data
into a training, a validation and a test set of sizes 14000,1000 and
4996, respectively.  We report MSE and accuracy on the test data in Table
\ref{tab:mse}. We
found that $k$-support regularization gave improved prediction
accuracy over both other methods.\footnote{Regarding other sparse
  prediction methods, we did not manage to compare with OSCAR, due to
  memory limitations, or to PEN or trace Lasso, which do not have code
  available online.}

\begin{table}[t] {\small\caption{\small Mean squared errors and
      classification accuracy for the synthetic data (median over $50$
      repetition), SA heart data (median over $50$
      replications) and for the ``$20$ newsgroups'' data set. (SE =
      standard error)}}
\label{tab:mse}
\begin{center}
{
\begin{tabular}{|c||c||c|c||c|c|}
\hline
  & Synthetic & \multicolumn{2}{c||}{Heart} & \multicolumn{2}{|c|}{Newsgroups} \\
\hline
Method & MSE (SE) & MSE (SE) & Accuracy (SE) & MSE & Accuracy \\
\hline
Lasso & 0.2746 (0.02) & 0.18 (0.005) & 66.41 (0.53) & 0.70 & 73.02 \\
Elastic net & 0.3119 (0.03) & 0.18 (0.005) & 66.41 (0.53)& 0.71 &
72.53 \\
$k$-support & {\bf 0.2342} (0.02) & 0.18 (0.005) & 66.41 (0.53)& {\bf
  0.69} & {\bf 73.40} \\
\hline
\end{tabular}
}
\end{center}
\end{table}


\section{Summary}

We introduced the $k$-support norm as the tightest convex relaxation of
sparsity plus $\ell_2$ regularization, and showed that it is tighter
than the elastic net by exactly a factor of $\sqrt{2}$.
In our view, this sheds light on the elastic net as a close
approximation to this tightest possible convex relaxation, and
motivates using the $k$-support norm when a tighter relaxation is
sought.  This is also demonstrated in our empirical results.

We note that the $k$-support norm has better prediction properties,
but not necessarily better sparsity-inducing properties, as evident
from its more rounded unit ball.  It is well understood that there is
often a tradeoff between sparsity and good prediction, and that even
if the population optimal predictor is sparse, a denser predictor
often yields better predictive performance \cite{elastic,oscar,jacob2009group}. 
For example, in the presence of correlated features, it is often
beneficial to include several highly correlated features rather than a
single representative feature.  This is exactly the behavior
encouraged by $\ell_2$ norm regularization, and the elastic net is
already known to yield less sparse (but more predictive) solutions.
The $k$-support norm goes a step further in this direction, often
yielding solutions that are even less sparse (but more predictive)
compared to the elastic net.

Nevertheless, it is interesting to consider whether compressed sensing
results, where $\ell_1$ regularization is of course central, can be
refined by using the $k$-support norm, which might be able to handle
more correlation structure within the set of features.


\bibliographystyle{plain}

\end{document}